\documentclass{article}

\usepackage{microtype}
\usepackage{graphicx}
\usepackage{subfigure}
\usepackage{booktabs} 

\usepackage{hyperref}

\usepackage[accepted]{icml2019} 

\usepackage{amssymb}
\usepackage[utf8]{inputenc}
\usepackage{times}
\usepackage{hyperref}
\usepackage{bm}
\usepackage{amsmath}
\usepackage{amsfonts}
\usepackage{amsthm}
\usepackage{graphicx,wrapfig}
\usepackage{subfiles}
\usepackage{enumitem}
\usepackage{verbatim}
\newtheorem{theorem}{Theorem}
\newtheorem{lemma}{Lemma}
\numberwithin{lemma}{section}
\newtheorem{assumption}{Assumption}

\numberwithin{theorem}{section}
\newtheorem{defn}{Definition}
\numberwithin{defn}{section}

\newtheorem{corollary}{Corollary}
\numberwithin{corollary}{section}
\numberwithin{prop}{section}

\numberwithin{myLemma}{section}

\usepackage{float}
\usepackage{epigraph}
\usepackage{dirtytalk}
\usepackage{mathtools}
\usepackage{dsfont}

\numberwithin{remark}{section}
\usepackage{thmtools}
\usepackage{thm-restate}
\usepackage{booktabs}
\usepackage{enumitem}
\usepackage{xparse}

\newcommand{\R}{\mathbb{R}}

\newcommand{\X}{\mathcal{X}}

\newcommand{\risk}{\mathcal{R}}

\newcommand{\Prob}{\mathbb{P}}

\newcommand{\N}{\mathbb{N}}
\newcommand{\E}{\mathbb{E}}

\newcommand{\one}{\mathds{1}}

\newcommand{\noisy}[1][]{
    \ifthenelse{\equal{#1}{Y}}
    {\tilde{Y}}
    {{#1}_{\text{corr}}}
    } 
\newcommand{\noisySample}{\noisy[\sample]}
\newcommand{\est}[1][]{\hat{#1}}
\newcommand{\KNNest}[1][f]{\hat{#1}_{n,k}}
\newcommand{\marginalDistribution}{\mu}
\newcommand{\regressionFunction}{\eta}
\newcommand{\noisyRegressionFunction}{\noisy[\regressionFunction]}
\newcommand{\estNoisyRegressionFunction}{\est[\regressionFunction]\noisy[]}

\newcommand{\probFlipFrom}[1][y]{p_{#1}}
\newcommand{\probFlipZeroToOne}{\probFlipFrom[0]}
\newcommand{\probFlipOneToZero}{\probFlipFrom[1]}

\newcommand{\estProbFlipZeroToOne}{\est[p]_0}
\newcommand{\estProbFlipOneToZero}{\est[p]_1}

\newcommand{\Y}{\mathcal{Y}}

\newcommand{\dist}{\rho}
\newcommand{\classifier}{\phi}
\newcommand{\oracle}{\classifier_*}
\newcommand{\sample}{\mathcal{D}}
\newcommand{\Xsample}{\mathbf{X}}

\newcommand{\Zsample}{\mathbf{Z}}
\newcommand{\conditionalExpectKNNest}[1][x]{\tilde{#1}_{n,k}}

\newcommand{\measureSmoothnessExponent}{\lambda}

\newcommand{\suppMarginalDistribution}{\X_{\marginalDistribution}}
\newcommand{\measureSmoothnessConstant}{\omega}

\newcommand{\lepskiChoiceOfKBase}[1][k]{\hat{k}_{n}^{\delta}(x)}
\newcommand{\lepskiChoiceOfK}[1][k]{\hat{k}_{n}^{\delta}}
\newcommand{\KNNOptimalEstBase}[3]{\widehat{#1}_{#2,#3}}

\newcommand{\empiricalMaximumBase}[2]{{\widehat{M}_{#2}\left(#1\right)}}
\newcommand{\empiricalMaximum}[1][f]{\empiricalMaximumBase{#1}{n,k}}

\newcommand{\marginExponent}{\alpha}
\newcommand{\marginConstant}{C_{\marginExponent}}

\newcommand{\XFocusPoint}{X}
\newcommand{\excessRisk}{\mathcal{E}}

\icmltitlerunning{
Fast Rates for a kNN Classifier Robust to Unknown Asymmetric Label Noise
}
\begin{document}

\twocolumn[
\icmltitle
{Fast Rates for a kNN Classifier Robust to Unknown Asymmetric Label Noise}



\icmlsetsymbol{equal}{*}

\begin{icmlauthorlist}
\icmlauthor{Henry W. J. Reeve}{to}
\icmlauthor{Ata Kab\'an}{to}
\end{icmlauthorlist}

\icmlaffiliation{to}{University of Birmingham, UK}
\icmlcorrespondingauthor{Henry W. J. Reeve}{henrywjreeve@gmail.com}

\icmlkeywords{Label noise, k nearest neighbours, fast rates, non-parametric, metric spaces}

\vskip 0.3in
]



\printAffiliationsAndNotice{}  

\begin{abstract}
We consider classification in the presence of class-dependent asymmetric label noise with unknown noise probabilities. In this setting, identifiability conditions are known, but additional assumptions were shown to be required for finite sample rates, and so far only the parametric rate has been obtained. Assuming these identifiability conditions, together with a measure-smoothness condition on the regression function and Tsybakov’s margin condition, we show that the Robust kNN classifier of Gao et al. attains, the mini-max optimal rates of the \emph{noise-free setting}, up to a log factor, even when trained on data with unknown asymmetric label noise. Hence, our results provide a solid theoretical backing for this empirically successful algorithm. By contrast the standard kNN is not even consistent in the setting of asymmetric label noise. A key idea in our analysis is a simple kNN based method for estimating the maximum of a function that requires far less assumptions than existing mode estimators do, and which may be of independent interest for noise proportion estimation and randomised optimisation problems.
\end{abstract}

\section{Introduction}\label{introductionSec}
Label noise is a pervasive issue in real-world classification tasks,
as perfectly accurate labels are often very costly, and sometimes impossible, to produce \cite{natarajan2018,cannings2018,blanchard2016,Frenay1}.

We consider asymmetric label noise with unknown class-conditional noise probabilities -- that is, the labels we observe have randomly flipped in some proportion that depends on the class. This type of noise 
is both realistic and amenable to analysis \cite{blanchard2016,natarajan2018}. In this setting the classical kNN algorithm is no longer consistent (see Section \ref{inconsistencySec}). Most existing theoretical work in this direction assumes that the noise probabilities are known in advance by the learner \cite{natarajan2013learning}, 
at least approximately \cite{natarajan2018}. However, in many situations such knowledge is not available, for instance in positive unlabelled (PU) learning \cite{ElkanNoto} one may regard unlabelled data as a class of negative examples contaminated with positives in an unknown proportion. Other examples include the problem of nucleur particle classification discussed by \cite{blanchard2016}. That work also established identifiability conditions sufficient for recovering unknown noise probabilities from corrupted data. 

\citet{blanchard2010semi} proved that the identifiability conditions are insufficient to obtain finite sample convergence rates. Consequently,  \citet{scott2013classification} introduced additional conditions external to the classification task with which it is possible to obtain the parametric rate (of order $n^{-1/2}$ where $n$ is the sample size) \cite{blanchard2016}. To the best of our knowledge it is unknown if faster rates are possible with unknown asymmetric label noise.

Here we answer this question in the affirmative by analysing an existing Robust kNN classifier \cite{Gao}. Previously, \citet{Gao} conducted a comprehensive empirical study which demonstrates that the Robust kNN, introduced therein, typically outperforms a range of competitors for classification problems with asymmetric label noise. \citet{Gao} also proved the consistency of their method, but only under the restrictive assumption of prior knowledge of the label noise probabilities. 

We prove that the Robust kNN classifier attains fast rates for classification problems  in a flexible non-parametric setting with unknown asymmetric label noise. More precisely, we work under 
a measure-smoothness condition on the regression function, introduced in recent analyses of kNN in the noise-free setting \cite{chaudhuri2014rates}, termed the `modified Lipschitz' condition in \cite{doring2018}, as well as Tsybakov's margin condition. We assume in addition conditions equivalent to those of label noise identifiability \cite{blanchard2016, menon2015learning}. We show that the Robust kNN introduced by \cite{Gao} attains, up to a log factor, the known minimax optimal fast rate of the label noise free setting -- despite the presence of unknown asymmetric label noise.


\section{Problem Setup}\label{problemStatementSec}
Suppose we have a feature space $\X$ with a metric $\dist$ and a set of labels $\Y = \{0,1\}$. Let $\Prob$ be a fixed but unknown distribution on $\X \times \Y$.  Our goal is to learn a classifier $\classifier:\X \rightarrow \Y$ which minimises the risk
\begin{align*}
\risk\left(\classifier\right):= \E\left[ \classifier(X) \neq Y \right].
\end{align*}
Our data will be generated by a \emph{corrupted} distribution $\noisy[\Prob]$ on $\X\times \Y$ with asymmetric label noise, so there exist probabilities $\probFlipZeroToOne, \probFlipOneToZero \in (0,1)$ with $\probFlipZeroToOne+\probFlipOneToZero<1$ such that random pairs $(X,\noisy[Y])\sim \noisy[\Prob]$ are generated by $(X,Y) \sim \Prob$ and $\noisy[Y]\neq Y$ with probability $\probFlipFrom[Y]$ and $\noisy[Y]=Y$ otherwise, i.e. $\probFlipZeroToOne = \noisy[\Prob][\noisy[Y]=1|Y=0]$ and $\probFlipOneToZero = \noisy[\Prob][\noisy[Y]=0|Y=1]$. We have access to a data set $\noisySample = \{ (X_i,\noisy[Y]_i)\}_{i \in [n]}$ only, consisting of i.i.d. pairs generated from the corrupted distribution $(X_i,\noisy[Y]_i) \sim \noisy[\Prob]$.

We let $\marginalDistribution$ denote the  marginal distribution over the features i.e. $\marginalDistribution(A) = \Prob\left[ X \in A \right]$ for Borel sets $A \subseteq \X$, and let $\regressionFunction:\X\rightarrow [0,1]$ denote the regression function i.e. $\regressionFunction(x) = \Prob[ Y=1 | X=x]$. Further, let $\suppMarginalDistribution\subseteq \X$ denote the support of the measure $\marginalDistribution$. It follows from the assumption of feature independent label noise that the corrupted distribution $\noisy[\Prob]$ has the same marginal distribution as $\Prob$ i.e. $\noisy[\Prob]\left[ X \in A \right]= \marginalDistribution\left(A\right)$ for $A \subseteq \X$. Denote by $\noisy[\regressionFunction]:\X\rightarrow [0,1]$ the corrupted regression function $\noisy[\regressionFunction](x)  = \noisy[\Prob][\noisy[Y]=1|X=x]$. As observed in 
\cite{menon2015learning}, $\noisy[\regressionFunction]$ and $\regressionFunction$ are related by
{\small
\begin{align}\label{noisyRegInTermsOfTrueRegEq}
\noisy[\regressionFunction](x) &= \left(1-\probFlipOneToZero\right) \cdot \Prob\left[Y=1|X=x\right] + \probFlipZeroToOne \cdot \Prob\left[Y=0|X=x\right]\nonumber\\
&=\left(1-\probFlipZeroToOne-\probFlipOneToZero\right) \cdot \regressionFunction(x)+\probFlipZeroToOne.
\end{align}
}
We shall use this connection to provide a label noise robust plug-in classifier. 

\section{Approach -- roadmap}\label{plugInCorruptedRegressionFunctionSec}
The `plug-in' classification method is inspired by the fact that the mapping $\oracle:\X\rightarrow \Y$ for $x \in \X$ defined by $\oracle(x) = \one\left\lbrace \regressionFunction(x)\geq 1/2\right\rbrace$ is a Bayes classifier and minimises the risk $\risk(\classifier)$ over all measurable classifiers $\classifier:\X\rightarrow\Y$. The approach is to first produce an estimate $\est[\regressionFunction]:\X\rightarrow [0,1]$ of the regression function $\regressionFunction$ and then take $\est[\classifier](x):=\one\left\lbrace \est[\regressionFunction](x)\geq 1/2\right\rbrace$. To apply this method in the label-noise setting we must first give a method for constructing an estimate $\est[\regressionFunction]$ based upon the corrupted sample $\noisy[\sample]$. By eq. (\ref{noisyRegInTermsOfTrueRegEq}) for each $x \in \X$ we have
\begin{align}\label{trueRegInTermsOfNoisyRegEq}
\regressionFunction(x)=\left(1-\probFlipZeroToOne-\probFlipOneToZero\right)^{-1}\cdot\left(\noisy[\regressionFunction](x)-\probFlipZeroToOne\right).
\end{align}
However, all quantities on the RHS are unknown. Our strategy is to decompose the problem under mild conditions, so that we can plug in estimates. The following simple lemma makes this precise.

\begin{lemma}\label{elementaryRatioLemmaForCorruptedRegressionFunctionEstimation} 
 Let $\noisy[{\est[\regressionFunction]}]:\X \rightarrow [0,1]$ be an estimate of $\noisyRegressionFunction$ and define $\est[\regressionFunction] :\X \rightarrow [0,1]$ by $\est[\regressionFunction](x):= \left(\noisy[{\est[\regressionFunction]}](x)-\estProbFlipZeroToOne\right)/\left(1-\estProbFlipZeroToOne-\estProbFlipOneToZero\right)$.
Suppose that $\probFlipZeroToOne+\probFlipOneToZero<1$, and 
$\estProbFlipZeroToOne, \estProbFlipOneToZero \in \left[0,1\right)$ with $\estProbFlipZeroToOne+\estProbFlipOneToZero<1$. Suppose further that $\max\left\lbrace\left| \estProbFlipZeroToOne-\probFlipZeroToOne \right|, \left| \estProbFlipOneToZero-\probFlipOneToZero \right| \right\rbrace  \leq \left(1-\probFlipZeroToOne-\probFlipOneToZero\right)/4$. Then for all $x \in \X$ we have
\begin{align*}
&\left| \est[\regressionFunction](x)-\regressionFunction(x)\right| \leq  \dots \\
&8\cdot \frac{\max\left\lbrace \left| \estNoisyRegressionFunction(x) - \noisyRegressionFunction(x)\right|, \left| \estProbFlipZeroToOne-\probFlipZeroToOne \right|, \left| \estProbFlipOneToZero-\probFlipOneToZero \right| \right\rbrace}
{1-\probFlipZeroToOne-\probFlipOneToZero}.
\end{align*}
\end{lemma}
\begin{proof}The lemma follows from eq. (\ref{noisyRegInTermsOfTrueRegEq}) by a straightforward manipulation. See Appendix \ref{proofOfelementaryRatioLemmaForCorruptedRegressionFunctionEstimationAppendix} for details.
\end{proof}
We note that the conditions in Lemma \ref{elementaryRatioLemmaForCorruptedRegressionFunctionEstimation} that involve estimates may be ensured by access to a sufficiently large sample, and are therefore not restrictive. 

Consequently, we shall obtain $\hat{\eta}(x)$ in two steps, summarised in the plug-in template Algorithm \ref{classificationBasedOnCorruptedRegressionFunctionAlgo}.
First we construct an estimator $\noisy[{\est[\regressionFunction]}]$ for the corrupted regression function $\noisy[\regressionFunction]$ based upon the corrupted sample $\noisy[\sample]$ using supervised regression methods. The key remaining challenge is then to obtain estimates $\estProbFlipZeroToOne$ and $\estProbFlipOneToZero$ for $\probFlipZeroToOne$ and $\probFlipOneToZero$, respectively. The latter is known to be impossible without further assumptions (see Section 4 in \cite{scott2013classification}). 
\begin{algorithm}[htbp]
 {\caption{Plug-in classification with label noise \label{classificationBasedOnCorruptedRegressionFunctionAlgo}}}
{ 
 \begin{enumerate}
     \item Compute an estimate $\noisy[{\est[\regressionFunction]}]$ of the corrupted regression function $\noisy[\regressionFunction]$ based on $\noisy[\sample]$;
     \item Compute $\estProbFlipZeroToOne$ and $\estProbFlipOneToZero$ by estimating the extrema of $\noisy[\hat{\regressionFunction}]$;
     \item Let $\est[\classifier](x):= \one\left\lbrace \noisy[{\est[\regressionFunction]}](x) \geq 1/2 \cdot \left(1+\estProbFlipZeroToOne-\estProbFlipOneToZero\right) \right\rbrace$.
     \end{enumerate}
 }
\end{algorithm}
Next, we discuss the assumptions that we employ for the remainder of this work.

\subsection{Main Assumptions and Relation to Previous Work
}\label{assumptions}

We employ two kinds of assumptions: (i) Assumptions \ref{majorityAssumption} and \ref{rangeAssumption} represent identifiability conditions for asymmetric label noise; (ii) Assumptions \ref{regressionFunctionIsUniformlySmoothAssumption} and \ref{marginAssumption} are conditions under which minimax optimal fast rates are known in the noise-free setting. We now briefly explain each of these, which also serves to place our forthcoming analysis into the context of previous work.

We already made use of the following in Lemma \ref{elementaryRatioLemmaForCorruptedRegressionFunctionEstimation}:
\begin{assumption}[Most labels are correct]\label{majorityAssumption}  $\probFlipZeroToOne+\probFlipOneToZero<1$.
\end{assumption}
\begin{assumption}[Range assumption]\label{rangeAssumption} We have $\inf_{x \in \suppMarginalDistribution}\left\lbrace \regressionFunction(x) \right\rbrace=0$ and $\sup_{x \in \suppMarginalDistribution}\left\lbrace \regressionFunction(x) \right\rbrace=1$.
\end{assumption}
Assumption \ref{rangeAssumption} was introduced by \citet{menon2015learning} who showed it to be equivalent to the `mutual irreducibility' condition given in \cite{scott2013classification,blanchard2016}. The above form will be more directly useful, since 
from Assumption \ref{rangeAssumption} and eq. (\ref{noisyRegInTermsOfTrueRegEq}) it follows that $\inf_{x \in \suppMarginalDistribution}\left\lbrace \noisy[\regressionFunction](x) \right\rbrace=\probFlipZeroToOne$ and $\sup_{x \in \suppMarginalDistribution}\left\lbrace \noisy[\regressionFunction](x) \right\rbrace=1-\probFlipOneToZero$. Hence, we may obtain estimates $\estProbFlipZeroToOne$ and $\estProbFlipOneToZero$ by estimating the extrema of the corrupted regression function $\noisy[\regressionFunction]$. 

Recall that the above assumptions alone do not permit finite sample convergence rates; therefore \citet{scott2015rate} assumed in addition that there are positive measure balls $B_0,B_1$ in the input space such that $\forall x \in B_i$ $\regressionFunction(x)= i$, and obtained the parametric rate, i.e. of order $n^{-1/2}$.
We will not assume this, instead we now consider assumptions that have already succeeded in establishing fast rates in the noise-free setting.

The following smoothness condition with respect to the marginal distribution was first proposed by \citet{chaudhuri2014rates}, and further employed by \citet{doring2018} who also termed it the `modified Lipschitz' condition. It generalises the combination of H\"older-continuity and strong density assumptions (i.e. marginal density bounded from below) that have been prevalent in previous theoretical analyses of plug-in classifiers after \cite{audibert2007fast}.
\begin{defn}[Measure-smoothness]\label{def:GMS} 
A function $f:\X \rightarrow [0,1]$ is \emph{measure-smooth} with exponent $\measureSmoothnessExponent>0$ and constant $\measureSmoothnessConstant>0$ if for all $x_0, x_1 \in \X$ we have $
\left|f(x_0)-f(x_1)\right|\leq \measureSmoothnessConstant \cdot \marginalDistribution\left(B_{\rho(x_0,x_1)}(x_0)\right)^{\measureSmoothnessExponent}$.
\end{defn}
\begin{assumption}[Measure-smooth regression function]\label{regressionFunctionIsUniformlySmoothAssumption} The regression function $\regressionFunction$ is measure-smooth with exponent $\measureSmoothnessExponent$ and constant $\measureSmoothnessConstant$.
\end{assumption}
A sufficient condition for Assumption \ref{regressionFunctionIsUniformlySmoothAssumption} to hold with $\measureSmoothnessExponent = \beta/d$ is that $\regressionFunction$ is $\beta$-H\"{o}lder and $\marginalDistribution$ is absolutely continuous with respect to the Lebesgue measure on $[0,1]^d$ with a uniform lower bound on the density. However, Assumption \ref{regressionFunctionIsUniformlySmoothAssumption} does not require the existence of a density for $\marginalDistribution$ and also applies naturally to classification in general metric spaces, including discrete distributions \cite{doring2018}. 

The final assumption is Tsybakov's margin condition, which has become a widely used device for explaining fast rates, since it was first proposed by \citet{mammen1999}.
\begin{restatable}[Tsybakov margin condition]{assumption}{marginAssumption}\label{marginAssumption} There exists $\marginExponent\geq 0$ and $\marginConstant \geq 1$ such that for all $\xi>0$ we have
\begin{align*}
\marginalDistribution\left( \left\lbrace x \in \X: 0< \left|\regressionFunction(x)-\frac{1}{2}\right| <\xi \right\rbrace \right) \leq \marginConstant \cdot \xi^{\marginExponent}.
\end{align*}
\end{restatable}
Note that Assumption \ref{marginAssumption} always holds with $\marginExponent=0$ and $\marginConstant = 1$; hence it is not restrictive.

Under the assumptions of Tsybakov margin, and measure-smoothness of the regression function, \citet{chaudhuri2014rates} obtained, in the noise-free setting (i.e. without label noise), for the k nearest neighbor (kNN) classifier, the convergence rate of order $n^{-\frac{\measureSmoothnessExponent(\marginExponent+1)}{2\measureSmoothnessExponent+1}}$ -- which corresponds (after having made explicit the dimensional dependence) to the minimax optimal rate computed by \citet{audibert2007fast} (that is, the lower bound is over all classifiers). With these two distributional assumptions, this rate can therefore be regarded as quantifying the statistical hardness of classification in the minimax sense (see e.g. \cite{Tsybakov}) when a perfect labelling is available. It is not at all obvious whether, and how, the same rate can be achieved in the presence of unknown asymmetric label noise conditions? The aim in the remainder of this paper is to answer this question.

\subsection{Notation and tools}\label{preliminaries}
Whilst we are motivated by the estimation of $\noisy[\regressionFunction]$ we shall frame our results in a more general fashion for clarity. Suppose we have a distribution $\Prob$ on $\X \times [0,1]$ and let $f:\X\rightarrow [0,1]$ be the function $f(x):= \E\left[Z|X=x\right]$. Our goals are to estimate $f$ and its extrema based on a sample $\sample_f = \left\lbrace \left(X_i,Z_i\right)\right\rbrace_{i \in [n]}$ with $(X_i,Z_i) \sim \Prob$ generated i.i.d. 
We let $\Xsample: = \left\lbrace X_i\right\rbrace_{i \in [n]}$ and $\Zsample:=\left\lbrace Z_i\right\rbrace_{i \in [n]}$.

Given $x \in \X$ we define $\left\lbrace \tau_{n,q}(x)\right\rbrace_{q \in [n]}$ to be an enumeration of $[n]$ such that for each $q \in [n-1]$, $\rho\left(x,X_{\tau_{n,q}(x)}\right) \leq \rho\left(x,X_{\tau_{n,q+1}(x)}\right)$.
We define the $k$-nearest neighbour estimate $\KNNest[f]: \X \rightarrow [0,1]$ of $f$ by 
\begin{align}\label{kNNRegFuncDef}
\KNNest[f](x):=\frac{1}{k} \cdot \sum_{q \in [k]}Z_{\tau_{n,q}(x)}.
\end{align}
Given a point $x \in \X$ and $r>0$ we let $B_r(x)$ denote the open metric ball of radius $r$, centered at $x$. 
It will be useful to give a high probability bound on the measure of an open metric ball centered at a given point with random radius equal to the distance of its $k$-th nearest neighbour.
\begin{restatable}{lemma}{MeasureOfBallLemma}\label{MeasureOfBallLemma}
Take $x \in \X, k \in [n]$ and $\zeta\geq 0$. 
Then, 
{\small
\begin{align*}
\Prob^n\left[ \marginalDistribution\left(B_{\rho\left(x,X_{\tau_{n,k}(x)}\right)}(x)\right) > 
\frac{(1+\zeta)k}{n} 
\right] \leq  
e^{-k(\zeta-\log(1+\zeta))}.
 \end{align*}}
\end{restatable}
A bound of this form appears in \cite{Biau} (Sec. 1.2) for the special case where the marginal distribution has a continuous distribution function. Their proof relies of the fact that, in this special case  $\mu(B_{\rho\left(x,X_{\tau_{n,k}(x)}\right)}(x))$ follows the distribution of the $k$-th uniform order statistic whose properties are well studied. However, since we consider a general metric space setting and do not assume a continuous distribution function, below we show from first principles that this bound is still valid, by exploiting the continuity properties of measures.
\begin{proof}
For any $x\in \X$ and $p\in[0,1]$ we define (following \citet{chaudhuri2014rates}) the smallest radius for which the open ball centered at $x$ has probability at least $p$:
\begin{align*}
r_p(x) = \inf\left\lbrace r>0: \marginalDistribution\left(B_r(x)\right)\geq p\right\rbrace.
\end{align*}
Take $r>r_p(x)$, so $\mu\left(B_r(x)\right) \geq p$. 
Note that $\rho\left(x,X_{\tau_{n,k}(x)}\right) \geq r$ if and only if $\sum_{i\in [n]}\mathds{1}_{\{X_i \in B_r(x)\}}<k$. 
Moreover, taking 
\begin{align*}
\tilde{p}= \frac{1}{n}\sum_{i\in [n]}\E\left[\one_{\{X_i \in B_r(x)\}}\right]=\mu\left(B_r(x)\right) \geq p,
\end{align*}
implies $k \leq (1-\epsilon) n \tilde{p}\leq n \tilde{p}$ where $\epsilon = 1-{k}/{(np)}$. Thus, by the multiplicative Chernoff bound -- Theorem 4.5 in \cite{mitzenmacher2005probability} -- we have,
{\small
\begin{align*}
\Prob^n\left[\rho\left(x,X_{\tau_{n,k}(x)}\right) \geq r\right] &=\Prob^n\left[\sum_{i\in [n]}\mathds{1}_{\{X_i \in B_r(x)\}}<k\right]\\
&\leq \Prob^n\left[\sum_{i\in [n]}\mathds{1}_{\{X_i \in B_r(x)\}}< (1-\epsilon)n\tilde{p}\right]\\
&\leq \exp(-n\tilde{p} [\epsilon + (1-\epsilon)\log(1-\epsilon)])\\
& \leq \exp(-np [\epsilon + (1-\epsilon)\log(1-\epsilon)]).
\end{align*}}
Since the above inequality holds for all $r>r_p(x)$, it follows by continuity of $\mu$ from above that we have
{\begin{align*}
&\Prob^n\left[ \rho\left(x,X_{\tau_{n,k}(x)}\right) > r_p(x)\right] \leq 
e^{-np [\epsilon + (1-\epsilon)\log(1-\epsilon)]}.
 \end{align*}}
This implies that with probability at least 
$1-\exp(-np [\epsilon + (1-\epsilon)\log(1-\epsilon)])$ we have
\begin{align*}
 \mu\left(B_{\rho\left(x,X_{\tau_{n,k}(x)}\right)}(x)\right) &\leq \mu(B_{r_p(x)})\le p,
\end{align*}
where the last inequality follows by continuity of measure from below.

Recall that $\epsilon=1-k/(np)$. To obtain the conclusion of the lemma we first note that the bound holds trivially whenever $\zeta \geq n/k - 1$ since this implies
{\small
\begin{align*}
\Prob^n&\left[ \marginalDistribution\left(B_{\rho\left(x,X_{\tau_{n,k}(x)}\right)}(x)\right) > 
\frac{(1+\zeta)k}{n} 
\right] \\ &
\leq  
\Prob^n\left[ \marginalDistribution\left(B_{\rho\left(x,X_{\tau_{n,k}(x)}\right)}(x)\right) > 
1 
\right]=0.
\end{align*}}
For $\zeta \in \left[0, n/k-1\right]$, we choose $p=(1+\zeta)k/n \in [0,1]$. Plugging into $\epsilon$ yields $\epsilon=\zeta/(1+\zeta)$, and rearranging the RHS of the probability bound completes the proof.
\end{proof}

When the centre of the metric ball is one of the data points, we have the following.
\begin{corollary}\label{MeasureOfBallCorr}
Take $k,j \in [n]$ and $\zeta>0$. 
Then, 
{\small
\begin{align*}
\Prob^n\left[ \mu(B_{\rho\left(X_j,X_{\tau_{n,k}(x)}\right)}(X_j)) > 
\frac{(1+\zeta)k}{n}
\right] \leq  
e^{-(k-1)(\zeta-\log(1+\zeta))}.
 \end{align*}}
\end{corollary}
Since this is a bound also for the ball with a non-random centre point, we may use it in both cases.
\begin{proof} 
Fix $X_j$ and apply Lemma \ref{MeasureOfBallLemma} to the $k-1$-th nearest neighbour in the remaining sample of size $n-1$. Expectation w.r.t. $X_j$ on both sides leaves the RHS unchanged. Then use that $\frac{k-1}{n-1} < \frac{k}{n}$.
\end{proof}

\section{Results}
We are now in a position to follow through the plan of our analysis. The following three subsections correspond directly to the three steps of the algorithm template (see Algorithm \ref{classificationBasedOnCorruptedRegressionFunctionAlgo}) through a kNN based classification rule.

\subsection{Pointwise function estimation}\label{KNNFuncEstSec}
As a first step we deal with point-wise estimation of the corrupted regression function, which we approach as a kNN regression task \cite{Kpotufe,Jiang}. However, for our purposes we require a bound that holds both for non-random points $x\in\X_{\mu}$ and for feature vectors $X_j$ occurring in the data, for reasons that will become clear in the subsequent Section \ref{KNNMaxEstBoundedSubSec} where we will need to estimate the maximum of the function. 
\begin{theorem}[Pointwise estimation bound]\label{KNNPointwiseBoundFixedKThm} Suppose that $f$ satisfies measure-smoothness with exponent $\measureSmoothnessExponent>0$ and constant $\measureSmoothnessConstant$. Take $n\in \N$, $\delta \in \left(0,1\right)$, $k \in \N\cap [4 \log(3/\delta)+1, n/2]$ and suppose that $\sample_f$ is generated i.i.d. from $\Prob$. Suppose that $\XFocusPoint$ is either a fixed point $x \in \suppMarginalDistribution$ or $X_j$ for some fixed $j \in [n]$. The following bound holds with probability at least $1-\delta$ over $\sample_f$
\begin{align*}
\left|\KNNest[f](\XFocusPoint)-f(\XFocusPoint)\right| \leq  \sqrt{\frac{\log(3/\delta)}{2k}}+\omega \cdot \left(\frac{2k}{n}\right)^{\measureSmoothnessExponent}.
\end{align*}

\end{theorem}
We leave $k$ unspecified for now, but we see that a choice of $k\in{\Theta}(\measureSmoothnessConstant^{-\frac{2}{2\measureSmoothnessExponent+1}}\cdot n^{\frac{2\measureSmoothnessExponent}{2\measureSmoothnessExponent+1}})$ gives the rate $\mathcal{O}(\measureSmoothnessConstant^{\frac{1}{2\measureSmoothnessExponent+1}}\cdot n^{-\frac{\measureSmoothnessExponent}{2\measureSmoothnessExponent+1}})$.

The following lemma will be handy.
\begin{lemma}\label{LipLemma}
Let $f, X, n, k$ as in Theorem \ref{KNNPointwiseBoundFixedKThm}, 
and $q\in [k]$. 
Then for any $\delta>0$, and $n/2\ge k\ge 4\log(1/\delta)+1$, we have w.p. $1-\delta$ that
\begin{align*}
\left|f\left(X_{\tau_{n,q}(\XFocusPoint)}\right)-f(\XFocusPoint)\right| 
&\leq  \measureSmoothnessConstant \cdot \left(\frac{2k}{n}\right)^{\measureSmoothnessExponent}.
\end{align*}
\end{lemma}
\begin{proof} 
By the measure-smoothness property,
\begin{align*}
\left|f\left(X_{\tau_{n,q}(\XFocusPoint)}\right)-f(\XFocusPoint)\right| &\leq \measureSmoothnessConstant \cdot \marginalDistribution\left(B_{\rho(\XFocusPoint,X_{\tau_{n,q}(x)})}(\XFocusPoint)\right)^{\measureSmoothnessExponent} \\& \leq \measureSmoothnessConstant \cdot \marginalDistribution\left(B_{\rho(\XFocusPoint,X_{\tau_{n,k}(\XFocusPoint)})}(\XFocusPoint)\right)^{\measureSmoothnessExponent} \\ 
&\leq  \measureSmoothnessConstant \cdot \left(\frac{(1+\zeta)k}{n}\right)^{\measureSmoothnessExponent},  
\end{align*}
with probability $1-\exp(-(k-1)(\zeta-\log(1+\zeta)))$, where $\zeta\geq 0$, and the last inequality follows by Corollary \ref{MeasureOfBallCorr}. 

For simplicity we choose $\zeta=1$, whence $1/(1-\log(2)) < 4$, so for $k\geq 4\log(1/\delta)+1$ we have with probability $1-\delta$ the statement of the Lemma.
Finally, note that the measure of the ball cannot be larger than $1$, so we require $k\leq n/2$.
\end{proof}

\begin{proof} 
We shall use the notation $\conditionalExpectKNNest[f](x)=\E_{\Zsample|\Xsample}\left[\KNNest[f](x)\right] =\frac{1}{k} \cdot \sum_{q \in [k]}f\left(X_{\tau_{n,q}(x)}\right)$. 

By the triangle inequality we have 
{\small
\begin{align}\label{mainTriangleInequalityInProofOfKNNPointwiseBoundFixedKThmEq}
\left|\KNNest[f](\XFocusPoint)-f(\XFocusPoint)\right| \leq \left|\KNNest[f](\XFocusPoint)-\conditionalExpectKNNest[f](\XFocusPoint)\right|+\left|\conditionalExpectKNNest[f](\XFocusPoint)-f(\XFocusPoint)\right|. 
\end{align}}
To bound the first term, note that $\XFocusPoint$ is a deterministic function of $\Xsample=\left\lbrace X_i \right\rbrace_{i \in [n]}$. By construction we have $\KNNest[f](\XFocusPoint) = \frac{1}{k} \cdot \sum_{q \in [k]}Z_{\tau_{n,q}(\XFocusPoint)}$. Moreover, for each $q \in [k]$, $Z_{\tau_{n,q}(\XFocusPoint)}$ is a random variable in $[0,1]$ with conditional expectation (given $\Xsample$) equal to $\E_{\Zsample|\Xsample}\left[Z_{\tau_{n,q}(\XFocusPoint)}\right] = f\left(X_{\tau_{n,q}(\XFocusPoint)}\right)$, and these are independent. 
Hence, it follows from Chernoff bounds \cite{boucheron2013concentration}
that the first term is bounded with probability at least $1-2\delta/3$ over $\sample_f$, as the following
\begin{align}\label{varBoundInProofOfKNNPointwiseBoundFixedKThmEq}
\left|\KNNest[f](\XFocusPoint)-\conditionalExpectKNNest[f](\XFocusPoint)\right|\leq \sqrt{\frac{\log(3/\delta)}{2k}}.
\end{align}

To bound the second term in eq. (\ref{mainTriangleInequalityInProofOfKNNPointwiseBoundFixedKThmEq}) 
we use Lemma \ref{LipLemma} with $1-\delta/3$, so for the allowed range of values of $k$ we have
{\small
\begin{align}\label{biasBoundInProofOfKNNPointwiseBoundFixedKThmEq}
\left|\conditionalExpectKNNest[f](\XFocusPoint)-f(\XFocusPoint)\right|&\leq \frac{1}{k} \cdot \sum_{q \in [k]}\left|f\left(X_{\tau_{n,q}(\XFocusPoint)}\right)-f(\XFocusPoint)\right| \\
& \leq \measureSmoothnessConstant \cdot \left(\frac{2k}{n}\right)^{\measureSmoothnessExponent}.  
\end{align}}
Taking the union bound, eqs. (\ref{varBoundInProofOfKNNPointwiseBoundFixedKThmEq}) and (\ref{biasBoundInProofOfKNNPointwiseBoundFixedKThmEq}) hold simultaneously with probability at least $1-\delta$. Plugging inequalities (\ref{varBoundInProofOfKNNPointwiseBoundFixedKThmEq}) and (\ref{biasBoundInProofOfKNNPointwiseBoundFixedKThmEq}) back into (\ref{mainTriangleInequalityInProofOfKNNPointwiseBoundFixedKThmEq}) completes the proof.
\end{proof}


\subsection{Maximum estimation with kNN}\label{KNNMaxEstBoundedSubSec}
In Section \ref{plugInCorruptedRegressionFunctionSec} we discussed how the noise probabilities $\probFlipZeroToOne$ and $\probFlipOneToZero$ are determined by the extrema of the corrupted regression function $\noisy[\regressionFunction]$. This motivates the question of determining the maximum of a function $f$, which is the focus of this section, although we believe the results of this section may also be of independent interest. As in Section \ref{KNNFuncEstSec}, we shall assume we have access to a sample $\sample_f = \left\lbrace \left(X_i,Z_i\right)\right\rbrace_{i \in [n]}$ with $(X_i,Z_i) \sim \Prob$ generated i.i.d. where $\Prob$ is an unknown distribution on $\X \times [0,1]$ with $f(x)= \E\left[Z|X=x\right]$. Our aim is to estimate $M(f):=\sup_{x \in \suppMarginalDistribution}\left\lbrace f(x) \right\rbrace$ based on $\sample_f$. We give a bound for a simple estimator under the assumption of measure-smoothness of the regression function.

Before proceeding we should point out that mode-estimation via kNN was previously proposed by \citep{DasKpo,jiang2017modal,Jiang}, but both solve a related but different problem to ours. The former papers deal with the unsupervised problem of finding the point where the density is highest. The latter work deals with finding the point where a function is maximal. The key difference is that performance is judged in terms of distance in the input space, whereas we care about distance in function output. As a consequence, previous works require strong curvature assumptions: That the Hessian exists and is negative definite for all modal points. By contrast, we are able to work on metric spaces where the notion of a Hessian does not even make sense, and we only require the measure-smoothness condition which holds, for instance, whenever the regression function is H\"older and its density is bounded from below. We also do not require a bounded input domain. 

Take the following estimator for $M(f)$, defined as the empirical maximum of the values of the regression estimator:
\begin{align*}
\empiricalMaximum[f] = \max_{i \in [n]}\left\lbrace \KNNOptimalEstBase{f}{n}{k}\left(X_i\right)\right\rbrace.
\end{align*}

\begin{theorem}[Maximum estimation bound with measure-smoothness]\label{NaiveMaxBoundThm} Suppose that $f$ is measure-smooth with exponent $\measureSmoothnessExponent>0$ and constant $\measureSmoothnessConstant>0$. 
 Take $n\in \N$, $\delta \in \left(0,1\right)$, $k \in \N\cap [4 \log(2/\delta)+1, n/2]$.
Suppose further that $\sample_f$ is generated i.i.d. from $\Prob$. 
Then for any $\delta \in \left(0,1\right)$ the following holds with probability at least $1-\delta$ over $\sample_f$,
\begin{align*}
\left| \empiricalMaximum[f]-M(f)\right| \leq   \sqrt{\frac{\log(6n/\delta)}{2k}}+2\omega \cdot \left(\frac{2k}{n}\right)^{\measureSmoothnessExponent}.
\end{align*}
\end{theorem}

\begin{proof} By
Theorem \ref{KNNPointwiseBoundFixedKThm} combined with the union bound, the following holds simultaneously for all $i \in [n]$, with probability at least $1-\delta/2$ over $\sample_f$,
{\small  
\begin{align}
\left| \KNNOptimalEstBase{f}{n}{k}(X_i)-f(X_i) \right| &\leq 
\sqrt{\frac{\log(6n/\delta)}{2k}}+\omega \cdot \left(\frac{2k}{n}\right)^{\measureSmoothnessExponent}. \label{knnBoundOverData}
\end{align}}
Given (\ref{knnBoundOverData}) we can upper bound $\empiricalMaximum[f]$ by
\begin{align*}
\empiricalMaximum[f] &\leq \max_{i \in [n]}\left\lbrace f(X_i) \right\rbrace + \sqrt{\frac{\log(6n/\delta)}{2k}}+\omega \cdot \left(\frac{2k}{n}\right)^{\measureSmoothnessExponent}\\
& \leq M(f)+  \sqrt{\frac{\log(6n/\delta)}{2k}}+2\omega \cdot \left(\frac{2k}{n}\right)^{\measureSmoothnessExponent}.
\end{align*}
We now lower bound $\empiricalMaximum[f]$ as follows. Take $\epsilon>0$ and choose $x_0 \in \suppMarginalDistribution$ with $f(x_0)\geq M(f) -\epsilon$ (a point that nearly achieves the supremum of $f$). By Lemma \ref{LipLemma} with  probability at least $1-\delta/2$ over $\sample_f$ we have  \begin{align}
|f\left(X_{\tau_{n,1}(x_0)}\right)- f(x_0)| 
\leq \measureSmoothnessConstant \cdot \left(\frac{2k}{n}\right)^{\measureSmoothnessExponent}. \label{lipsTermInPfOfMaxBoundEq}
\end{align} In conjunction, the bounds (\ref{knnBoundOverData}) and (\ref{lipsTermInPfOfMaxBoundEq}) imply
\begin{align*}
\empiricalMaximum[f] & \geq \hat{f}_{n,k}\left(X_{\tau_{n,1}(x_0)}\right)\\
& \geq f\left(X_{\tau_{n,1}(x_0}\right)-  \sqrt{\frac{\log(6n/\delta)}{2k}}-\omega \cdot \left(\frac{2k}{n}\right)^{\measureSmoothnessExponent}\\
& \geq f(x_0)-  \sqrt{\frac{\log(6n/\delta)}{2k}}-2\omega \cdot \left(\frac{2k}{n}\right)^{\measureSmoothnessExponent}\\
& \geq M(f)-  \sqrt{\frac{\log(6n/\delta)}{2k}}-2\omega \cdot \left(\frac{2k}{n}\right)^{\measureSmoothnessExponent}-\epsilon.
\end{align*}
Combining this with the upper bound we see that given  (\ref{knnBoundOverData}) and (\ref{lipsTermInPfOfMaxBoundEq}) we have
\begin{align}
\left| \empiricalMaximum[f]-M(f)\right|  &\leq   \sqrt{\frac{\log(6n/\delta)}{2k}}+2\omega \cdot \left(\frac{2k}{n}\right)^{\measureSmoothnessExponent} +\epsilon.\label{almostMaxThmEpsBound}
\end{align}
By the union bound applied to (\ref{knnBoundOverData}) and (\ref{lipsTermInPfOfMaxBoundEq}),  (\ref{almostMaxThmEpsBound}) holds with probability at least $1-\delta$. Letting $\epsilon\rightarrow 0$ and applying continuity of measure completes the proof of the theorem.
\end{proof}

\subsection{Main result: Fast rates in the presence of unknown asymmetric label noise}\label{mainBoundsSec}
We put everything together in this section, and complete the analysis of the label noise robust kNN classifier given in Algorithm \ref{KNNForLabelNoiseUniformCaseAlgo}. This classifier is simply the kNN instantiation of Algorithm \ref{classificationBasedOnCorruptedRegressionFunctionAlgo}. Algorithm \ref{KNNForLabelNoiseUniformCaseAlgo} was previously proposed by \citep{Gao} with an empirical demonstration of its success in practice, but without an analysis of its finite sample behaviour when the noise probabilities are unknown. We shall now prove that it attains the known minimax optimal rates \emph{of the noiseless setting}, up to logarithmic factors, despite the presence of unknown asymmetric label noise, provided the assumptions discussed in Section \ref{assumptions}.
\begin{algorithm}[htbp]
 {\caption{A k nearest neighbour method for label noise\label{KNNForLabelNoiseUniformCaseAlgo}}}
{ 
 \begin{enumerate}
     \item Define $\noisy[{\est[\regressionFunction]}]$ by $\noisy[{\est[\regressionFunction]}](x) :=\frac{1}{k} \cdot \sum_{q \in [k]}\noisy[Y]_{\tau_{n,q}(x)}$;
     
     \item Compute $\estProbFlipZeroToOne := \min_{i \in [n]}\left\lbrace \noisy[{\est[\regressionFunction]}](X_i)\right\rbrace$ and\\ \phantom{Compute} $\estProbFlipOneToZero: = 1-\max_{i \in [n]}\left\lbrace \noisy[{\est[\regressionFunction]}](X_i)\right\rbrace$;
     \item Let $\est[\classifier]_{n,k}(x):= \one\left\lbrace \noisy[{\est[\regressionFunction]}](x) \geq 1/2 \cdot \left(1+\estProbFlipZeroToOne-\estProbFlipOneToZero\right) \right\rbrace$.
     \end{enumerate}
 }
\end{algorithm}

\begin{theorem}\label{riskBoundForUniformSmoothnessThm} Suppose that Assumptions \ref{majorityAssumption} and \ref{rangeAssumption} hold, Assumption \ref{regressionFunctionIsUniformlySmoothAssumption} holds with exponent $\measureSmoothnessExponent>0$ and constant $\measureSmoothnessConstant>0$, and Assumption \ref{marginAssumption} holds with constant $\marginExponent\geq 0$ and $\marginConstant\geq 1$. Take any $n \in \N$, $\delta \in \left(0,1\right)$,
and suppose we have  $\noisySample = \{ (X_i,\noisy[Y]_i)\}_{i \in [n]}$ with $(X_i,\noisy[Y]_i) \sim \noisy[\Prob]$. Let $\est[\classifier]_{n,k}$ be the label-noise robust $k$NN classifier with an arbitrary choice of $k \in \N\cap [4 \log(3/\delta)+1, n/2]$ (Algorithm \ref{KNNForLabelNoiseUniformCaseAlgo}). With probability at least $1-\delta$ over $\noisy[\sample]$, we have
\begin{align}\label{riskBoundForUniformSmoothnessThmEq}
\risk\left(\est[\classifier]_{n,k}\right) &\leq \risk\left(\oracle\right)+
\marginConstant \cdot \left(\frac{8}{ 1-\probFlipZeroToOne-\probFlipOneToZero} \right)^{\marginExponent+1}
\nonumber\\
&\cdot \left[\sqrt{\frac{\log(18n/\delta)}{k}}+2\omega \cdot \left(\frac{2k}{n}\right)^{\measureSmoothnessExponent}\right]^{\marginExponent+1}
+\delta,
\end{align}
where $\oracle(x)\equiv\one\left\lbrace \regressionFunction(x)\geq 1/2\right\rbrace$ is the Bayes classifier. 

In particular, if we take 
$k_n=  
\left\lceil \left(\frac{\log(18n/\delta)}{
2\measureSmoothnessConstant^2} \right)^{\frac{1}{2\measureSmoothnessExponent+1}} \cdot n^{\frac{2\measureSmoothnessExponent}{2\measureSmoothnessExponent+1}} \right\rceil
$ then for $n\geq 5 \cdot (10 \cdot \measureSmoothnessConstant^2)^{\frac{1}{2\measureSmoothnessExponent}} \cdot \log(18n/\delta)$, w.p. at least $1-\delta$,
\begin{align*}
    \risk\left(\est[\classifier]_{n,k^*_n}\right) \leq \risk\left(\oracle\right) +\marginConstant &\cdot \left(\frac{2^{2\measureSmoothnessExponent+5}\cdot \measureSmoothnessConstant^{\frac{1}{2\measureSmoothnessExponent+1}} }{ 1-\probFlipZeroToOne-\probFlipOneToZero} \right)^{\marginExponent+1}\\&\cdot \left( \frac{\log(18n/\delta)}{n}\right)^{\frac{\measureSmoothnessExponent(\marginExponent+1)}{2\measureSmoothnessExponent+1}} +\delta.    
\end{align*}
\end{theorem}

\newcommand{\goodDeltaSet}{\mathcal{G}_{\delta}}
\newcommand{\smallErrors}{\xi(n,k,\delta)}
\begin{proof} 
Observe first that the measure-smoothness property of $\regressionFunction$ implies that $\noisy[\regressionFunction]$ is measure-smooth with the same exponent, and constant 
$(1-\probFlipZeroToOne-\probFlipOneToZero)\cdot \measureSmoothnessConstant \leq \measureSmoothnessConstant$, so we can work with the latter to avoid clutter.

We define the subset of the input domain where the corrupted regression function has low estimation error:
\[\goodDeltaSet := \left\lbrace x \in \suppMarginalDistribution: \left|\noisy[{\est[\regressionFunction]}](x)-\noisy[\regressionFunction](x)\right|\leq  \smallErrors \right\rbrace.\]
where $\smallErrors$ is a small error that will be made precise shortly.
We want to ensure that a randomly drawn test point is in this set with probability at least $1-\delta/3$. By Theorem \ref{KNNPointwiseBoundFixedKThm}, for each $x \in \suppMarginalDistribution$, 
the following holds with probability at least $1-\delta^2/3$,
{
\begin{align*}
\left|\noisy[{\est[\regressionFunction]}](x)-\noisy[\regressionFunction](x)\right|&\leq 
\sqrt{\frac{\log(3/(\delta^2/3))}{2k}}+\omega \cdot \left(\frac{2k}{n}\right)^{\measureSmoothnessExponent}\\
&= \sqrt{\frac{2\log(3/\delta)}{2k}}+\omega \cdot \left(\frac{2k}{n}\right)^{\measureSmoothnessExponent}\\
%
%
&=:\xi_1(n,k,\delta). 
%
\end{align*}}
That is, for each fixed $x \in \suppMarginalDistribution$, we have $x \in \goodDeltaSet$ with probability at least $1-\delta^2/3$ i.e. $\E_{\sample_f}\left[\one\left\lbrace x \notin \goodDeltaSet\right\rbrace \right] \leq \delta^2/3$. We now integrate over $\marginalDistribution$ and use Fubini's theorem as follows:
\begin{align*}
\E_{\sample_f}\left[\marginalDistribution\left(\suppMarginalDistribution\backslash \goodDeltaSet\right)\right] &= \E_{\sample_f}\left[\int \one\left\lbrace x \notin \goodDeltaSet\right\rbrace d\marginalDistribution(x) \right]  \nonumber\\&=\int \E_{\sample_f}\left[\one\left\lbrace x \notin \goodDeltaSet\right\rbrace \right]d\marginalDistribution(x)\leq \delta^2/3.      
\end{align*}
Thus, by Markov's inequality we have $\marginalDistribution\left(\suppMarginalDistribution\backslash \goodDeltaSet\right)\leq \delta$ with probability at least $1-\delta/3$. Furthermore, by Theorem \ref{NaiveMaxBoundThm} with probability at least $1-\delta/3$,
{\small
\begin{align*}
\left|\estProbFlipZeroToOne -\probFlipZeroToOne\right| \leq  
\sqrt{\frac{\log(6n/(\delta/3))}{2k}}+2\omega \cdot \left(\frac{2k}{n}\right)^{\measureSmoothnessExponent}
%
%
&=: \xi_2(n,k,\delta). 
\end{align*}}
Similarly, with  probability at least $1-\delta/3$ we have $\left|\estProbFlipOneToZero -\probFlipOneToZero\right| \leq \xi_2(n,k,\delta)$, and we let 
\begin{align}
\smallErrors &:=\max\{\xi_1(n,k,\delta),\xi_2(n,k,\delta)\} \nonumber\\ 
&\le  \sqrt{\frac{\log(18n/\delta)}{k}}+2\omega \cdot \left(\frac{2k}{n}\right)^{\measureSmoothnessExponent}. \label{smallErrors}
\end{align}

By the union bound, with probability at least $1-\delta$, we have $\marginalDistribution\left(\suppMarginalDistribution\backslash\goodDeltaSet\right)\leq \delta$ and $\max\left\lbrace \left|\estProbFlipZeroToOne -\probFlipZeroToOne\right| ,\left|\estProbFlipOneToZero -\probFlipOneToZero\right|\right\rbrace \leq \smallErrors$. 
Hence, it suffices to assume that $\marginalDistribution\left(\suppMarginalDistribution\backslash\goodDeltaSet\right)\leq \delta$ and $\max\left\lbrace \left|\estProbFlipZeroToOne -\probFlipZeroToOne\right| ,\left|\estProbFlipOneToZero -\probFlipOneToZero\right|\right\rbrace \leq \smallErrors$ holds and show that eq (\ref{riskBoundForUniformSmoothnessThmEq}) holds. 

Observe that we can rewrite $\est[\classifier]_{n}:\X\rightarrow \Y$ as $\est[\classifier]_{n}(x)=\one\left\lbrace \est[\regressionFunction](x)\geq 1/2\right\rbrace$, where $\est[\regressionFunction](x):= \left(\noisy[{\est[\regressionFunction]}](x)-\estProbFlipZeroToOne\right)/\left(1-\estProbFlipZeroToOne-\estProbFlipOneToZero\right)$. By Lemma \ref{elementaryRatioLemmaForCorruptedRegressionFunctionEstimation} for all $x \in \goodDeltaSet$ we have deterministically that: 
\begin{align*}
&\left| \est[\regressionFunction](x)-\regressionFunction(x)\right| \\&\leq 8 \cdot 
\frac{
\max\left\lbrace \left| \estNoisyRegressionFunction(x) - \noisyRegressionFunction(x)\right|, \left| \estProbFlipZeroToOne-\probFlipZeroToOne \right|, \left| \estProbFlipOneToZero-\probFlipOneToZero \right| \right\rbrace}{\left(1-\probFlipZeroToOne-\probFlipOneToZero\right)}\\ 
&\leq  8 \cdot \left(1-\probFlipZeroToOne-\probFlipOneToZero\right)^{-1} \cdot \smallErrors.
\end{align*}
Hence, observe that, given any $x \in \X$ with $\est[\classifier]_{n}(x)\neq \oracle(x)\equiv\one\left\lbrace \regressionFunction(x)\geq 1/2\right\rbrace$ we must have $\left|\regressionFunction(x)-1/2\right|\leq 8 \cdot \left(1-\probFlipZeroToOne-\probFlipOneToZero\right)^{-1} \cdot \smallErrors$. Hence, by Assumption \ref{marginAssumption}, with probability at least $1-\delta$, we have
{\small
\begin{align*}
&\risk\left(\est[\classifier]_{n,k}\right) -\risk\left(\oracle\right)\\ & = \int_{\X} \left|\regressionFunction(x)-\frac{1}{2}\right|\cdot \one\left\lbrace \est[\classifier]_{n,\delta}(x)\neq \oracle(x)\right\rbrace d\marginalDistribution(x)\\
&\leq \int_{\goodDeltaSet} \left|\regressionFunction(x)-\frac{1}{2}\right|\cdot \one\left\lbrace \est[\classifier]_{n,\delta}(x)\neq \oracle(x)\right\rbrace d\marginalDistribution(x)+\marginalDistribution\left(\X\backslash \goodDeltaSet\right)\\
&\leq  \int_{\X} \left|\regressionFunction(x)-\frac{1}{2}\right|\cdot \one\left\lbrace  \left|\regressionFunction(x)-\frac{1}{2}\right| \leq \frac{ 8 \cdot \smallErrors}{ 1-\probFlipZeroToOne-\probFlipOneToZero} \right\rbrace d\marginalDistribution(x)+\delta\\
&\leq \marginConstant \cdot \left(\frac{8 \cdot \smallErrors}{ 1-\probFlipZeroToOne-\probFlipOneToZero} \right)^{\marginExponent+1} +\delta.
\end{align*}}
Plugging in eq. (\ref{smallErrors}) completes the proof of the first part.

The second part follows by choosing $k$ that approximately equates the two terms on the right hand side of eq. (\ref{smallErrors}).
That is, with the choice
\begin{align*}
k_n=  
\left\lceil \left(\frac{\log(18n/\delta)}{
2\measureSmoothnessConstant^2} \right)^{\frac{1}{2\measureSmoothnessExponent+1}} \cdot n^{\frac{2\measureSmoothnessExponent}{2\measureSmoothnessExponent+1}} \right\rceil,
\end{align*}
given $n\geq 5 \cdot (10 \cdot \measureSmoothnessConstant^2)^{\frac{1}{2\measureSmoothnessExponent}} \cdot \log(18n/\delta)$ we have $k_n \geq 4\log(3/\delta)+1$ and $\xi(n,k_n,\delta)$ takes the form
\begin{align*}
    \xi(n,k_n,\delta)= 4^{\measureSmoothnessExponent+1}\cdot \measureSmoothnessConstant^{\frac{1}{2\measureSmoothnessExponent+1}} \cdot \left( \frac{\log(18n/\delta)}{n}\right)^{\frac{\measureSmoothnessExponent}{2\measureSmoothnessExponent+1}}.
\end{align*}
Plugging this into the excess risk completes the proof of the second part of the theorem.
\end{proof}

\subsubsection{On setting the value of $k$}
We used the theoretically optimal value of $k$ in our analysis, which is not available in practice. Methods exist to set $k$ in a data-driven manner. Cross-validation is amongst the most popular practical approaches \cite{Inouye,Gao}, and there is also an ample literature on adaptive methods (e.g. \citet[Chapter 8]{gine_nickl_2015}) that allow us to retain nearly optimal rates without access to the unknown parameters of the analysis. 

\section{Discussion: Inconsistency of kNN in the presence of asymmetric label noise}\label{inconsistencySec}
Our main result implies that under the measure-smoothness condition, and provided the label-noise identifiability conditions hold, the statistical difficulty of classification with or without asymmetric label noise is the same in the minimax sense, up to constants and log factors. However, the algorithm that we used to achieve this rate was not the classical kNN. This invites the question as to whether this must be so (or is it an artefact of the proof)? We find this question interesting in the light of observations and claims in the literature about the label-noise robustness of kNN and other high capacity models \cite{Tarlow,Gao} (see also the introductory section of \cite{MenonMLJ2018}).

To shed light on this, we show in this section that the classical kNN cannot achieve these rates, and even fails to be consistent in the presence of asymmetric label noise. In fact, in Theorem \ref{inconsistencyThm} below we shall see that \emph{any} algorithm that is Bayes-consistent in the classical sense may become inconsistent under this type of noise.
Indeed, on closer inspection, the robustness claims in the literature about kNN and other high capacity models \cite{Tarlow,Gao,MenonMLJ2018} -- explicitly or tacitly -- refer either to class-unconditional (i.e. symmetric) label noise, or assume that the regression function is bounded away from $1/2$. A recent result of \cite{cannings2018} even gives convergence rates for $k$-NN under instance-dependent label noise, but requires that the label noise probabilities become symmetric as the regression function approaches $1/2$.

In order to talk about consistency in the presence of label noise we need to make explicit the distinct roles of the train and test distributions in our notation. For any distribution $\Prob$ on $\X \times \{0,1\}$ and classifier $\classifier: \X \rightarrow \{0,1\}$ the risk is defined as $\risk\left(\classifier\right)=\risk\left(\classifier;\Prob\right):= \E\left[{\classifier}(X)\neq Y\right]$, where $\E$ denotes the expectation with respect to $\Prob$. In addition we let $\excessRisk\left(\classifier;\Prob\right)$ denote the excess risk defined by
$\excessRisk\left(\classifier,\Prob\right):=\risk(\classifier;\Prob) -\inf_{\tilde{\classifier}} \{ \risk(\tilde{\classifier};\Prob)\}$,
where the infimum is over all measurable functions $\tilde{\classifier}:\X \rightarrow \{0,1\}$. 
\begin{defn}[Consistency]\label{consistencyDefn} Let $\Prob_{\text{train}}$,$\Prob_{\text{test}}$ be probability distributions on $\X \times \{0,1\}$. Take $\tilde{\sample} = \{(X_i,\tilde{Y}_i)\}_{i\in \mathbb{N}}$ with $(X_i,\tilde{Y}_i)$ sampled independently from $\Prob_{\text{train}}$. For each $n \in \N$ we let $\hat{\classifier}_n$ denote the classifier obtained by applying the learning algorithm $\hat{\classifier}$ to the data set $\tilde{\sample}$. 
 We shall say that a classification algorithm $\hat{\classifier}$ is consistent with training distribution $\Prob_{\text{train}}$ and test distribution $\Prob_{\text{test}}$ if we have $\lim_{n \rightarrow \infty} \excessRisk(\hat{\classifier}_n;\Prob_{\text{test}}) = 0$
 almost surely
 over the data $\tilde{\sample}$.
\end{defn}
Let  $\Prob_{\left(\marginalDistribution,\regressionFunction\right)}$ denote the probability distribution on $\X \times \{0,1\}$ with marginal $\marginalDistribution$ and a regression function $\regressionFunction$. Learning with label noise means that the training distribution $\Prob_{\left(\marginalDistribution,\noisy[\regressionFunction]\right)}$ and test distribution $\Prob_{\left(\marginalDistribution,\regressionFunction\right)}$ are different.
We define the input set of disagreement:
{\small
\begin{align*}\label{defBayesConflictedSet}
\mathcal{A}_0(\regressionFunction,\noisy[\regressionFunction]):= \left\lbrace x \in \X: \left(\regressionFunction(x)-\frac{1}{2}\right)\left(\noisy[\regressionFunction](x)-\frac{1}{2}\right)<0\right\rbrace.
\end{align*}}
Note that $\mathcal{A}_0(\regressionFunction,\noisy[\regressionFunction])$ consists of points which are sufficiently close to the boundary with asymmetric label noise.

\begin{theorem}\label{inconsistencyThm} 
Suppose that a classification algorithm $\hat{\classifier}$ is consistent with 
$\Prob_{train}=\Prob_{test}=\Prob_{\left(\marginalDistribution,\noisy[\regressionFunction]\right)}$, and let $\regressionFunction\neq \noisy[\regressionFunction]$. 
If $\marginalDistribution\left(\mathcal{A}_0(\regressionFunction,\noisy[\regressionFunction]) \right)>0$ then $\hat{\classifier}$ is inconsistent with $\Prob_{train}=\Prob_{\left(\marginalDistribution,\noisy[\regressionFunction]\right)}$ and $\Prob_{test}=\Prob_{\left(\marginalDistribution,\regressionFunction\right)}$. 
\end{theorem} 
The proof is given in Appendix. The essence of the argument is the simple observation that if the regression functions of the training and testing distributions disagree, then the trained classifier cannot agree with both.
Below we give a family of examples on $\X=\R$ with class-conditional label noise, where the standard $k_n$-NN classifier is inconsistent, yet the $k_n$-NN method for asymmetric label noise (Algorithm \ref{KNNForLabelNoiseUniformCaseAlgo}) is consistent. 

\textbf{Example:} Take any $\probFlipZeroToOne, \probFlipOneToZero \in \left(0,1/2\right)$ with $p_0\neq p_1$ and let $m:=\left({2-3\probFlipZeroToOne-\probFlipOneToZero}\right)/\left({4\left(1-\probFlipZeroToOne-\probFlipOneToZero\right)}\right)$. It follows that $m \in \left(0,1\right)$. Let $\X=[0,1]$ and let $\marginalDistribution$ be the Lebesgue measure on $\X$. Define $\regressionFunction:\X \rightarrow [0,1]$ by
\begin{align*}
\regressionFunction(x):= \begin{cases} \frac{3x}{2} &\text{ if }x \in \left[0,\frac{2m}{3}\right]\\
m &\text{ if }x \in \left[\frac{2m}{3},\frac{2m+1}{3}\right]\\
\frac{3x-1}{2} &\text{ if }x \in \left[\frac{2m+1}{3},1\right].
\end{cases}
\end{align*}
A special case of this example, with $p_0=0.1$ and $p_1=0.3$ is depicted in Figure \ref{fig:my_label}.
\begin{figure}
    \centering
    \includegraphics[height=5.5cm,width=7.3cm]{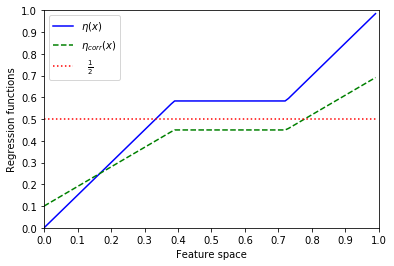}
    \caption{A pair of train and test regression functions ($\regressionFunction,\noisy[\regressionFunction]$) exemplifying a setting where classical kNN is inconsistent yet the kNN method for asymmetric label noise is consistent. }
    \label{fig:my_label}
\end{figure}

Indeed, for this family of examples, it follows from Theorem 1 in \cite{chaudhuri2014rates} that the standard $k_n$-NN classifier is strongly Bayes-consistent with $\Prob_{train}=\Prob_{test}=\Prob_{\left(\marginalDistribution,\noisy[\regressionFunction]\right)}$ whenever $k_n/n \rightarrow 0$ and $k_n/(\log(n)) \rightarrow \infty$ as $n \rightarrow \infty$. Moreover, it follows from the definition of $m$, $p_0\neq p_1$, and eq. (\ref{noisyRegInTermsOfTrueRegEq}) that for $x \in \left[\frac{2m}{3},\frac{2m+1}{3}\right]$ we have
{\small
\begin{align*}
\left(\regressionFunction(x)-\frac{1}{2}\right)\left(\noisy[\regressionFunction](x)-\frac{1}{2}\right)<0.
\end{align*}}
Hence, $\marginalDistribution\left(\mathcal{A}_0(\regressionFunction,\noisy[\regressionFunction])\right) = 1/3>0$. Thus, by Theorem \ref{inconsistencyThm}, the $k_n$-NN classifier is inconsistent with $\Prob_{train}=\Prob_{\left(\marginalDistribution,\noisy[\regressionFunction]\right)}$ and $\Prob_{test}=\Prob_{\left(\marginalDistribution,{\regressionFunction}\right)}$. 
On the other hand, one can readily check that Assumptions \ref{rangeAssumption} and \ref{majorityAssumption} hold. Moreover, Assumption \ref{regressionFunctionIsUniformlySmoothAssumption} holds with exponent $\measureSmoothnessExponent=1$ and constant $\measureSmoothnessConstant = 3$, and Assumption \ref{marginAssumption} holds with exponent $\marginExponent=0$ and constant $\marginConstant=1$. Thus, by Theorem \ref{riskBoundForUniformSmoothnessThm} combined with the Borel-Cantelli lemma, the $k_n$ method for asymmetric label noise (Algorithm \ref{KNNForLabelNoiseUniformCaseAlgo}) is consistent with $\Prob_{train}=\Prob_{\left(\marginalDistribution,\noisy[\regressionFunction]\right)}$ and $\Prob_{test}=\Prob_{\left(\marginalDistribution,{\regressionFunction}\right)}$ whenever $k_n/n \rightarrow 0$ and 
when $k_n/(\log(n)) \rightarrow \infty$ as $n \rightarrow \infty$.

\section{Conclusions}
We obtained fast rates in the presence of unknown asymmetric label noise that match the minimax optimal rates of the noiseless setting, up to logarithmic factors, under measure-smoothness and Tsybakov margin assumptions. On the practical side, our results provide theoretical support for the Robust kNN algorithm of \cite{Gao} whose analysis so far only exists under known noise probabilities. On the theoretical side, our results entail that under the stated conditions the statistical difficulty of classification with or without unknown asymmetric label noise is the same in the minimax sense. This is especially interesting given recent results which show that under more general non-parametric settings the optimal rates for unknown asymmetric label noise can be strictly slower than those for the noiseless case \cite{reeve2019COLT}. We have also seen that the algorithm achieving the rate in the presence of unknown asymmetric label noise must be different from any classical Bayes-consistent classifier, as those fail to be consistent under the label noise. Finally, a key ingredient in our analysis a simple method for estimating the maximum of a function that requires far less assumptions than existing mode estimators do and may have wider applicability.

\appendix

\subsection*{Acknowledgement} 
This work is funded by EPSRC under Fellowship grant EP/P004245/1, and a Turing Fellowship (grant
EP/N510129/1). We would also like to thank the anonymous reviewers for their careful feedback, and Joe Mellor for improving the presentation.



\newpage
\onecolumn


\appendix


\section{Proof of Theorem \ref{inconsistencyThm}}

The proof of Theorem \ref{inconsistencyThm} is as follows.

\begin{proof} 
Define the input set of disagreement with margin $\theta$:
{
\begin{align}
\mathcal{A}_{\theta}(\regressionFunction,\noisy[\regressionFunction]) &:= \left(\left\lbrace \regressionFunction(x)\leq \frac{1}{2}-\theta  \right\rbrace\cap \left\lbrace \noisy[\regressionFunction](x)\geq \frac{1}{2}+\theta  \right\rbrace\right)\\ &\cup  \left(\left\lbrace \regressionFunction(x)\geq \frac{1}{2}+\theta  \right\rbrace\cap \left\lbrace \noisy[\regressionFunction](x)\leq \frac{1}{2}-\theta  \right\rbrace\right).
\label{thetaDisagreement}
\end{align}}
We can write $\mathcal{A}_0(\regressionFunction,\noisy[\regressionFunction])$
as a union of such sets:
$\mathcal{A}_0(\regressionFunction) = \bigcup_{\theta>0}\mathcal{A}_{\theta}(\regressionFunction)$, and hence
\begin{align*}
\lim_{\theta \rightarrow 0}\left\lbrace \marginalDistribution\left( \mathcal{A}_{\theta}(\regressionFunction,\noisy[\regressionFunction])\right) \right\rbrace = \marginalDistribution\left(\mathcal{A}_0(\regressionFunction,\noisy[\regressionFunction]) \right)>0.
\end{align*}
Now, take some $\theta>0$ s.t. $\mathcal{A}_0(\regressionFunction,\noisy[\regressionFunction])>0$.
Lemma \ref{inconsistencyLemma} below will show that
\begin{align}\label{ineqApplyingInconsistencyLemma}
\excessRisk\left({\classifier};\Prob_{\left(\marginalDistribution,\regressionFunction\right)}\right)+\excessRisk\left({\classifier};\Prob_{\left(\marginalDistribution,\noisy[\regressionFunction]\right)}\right) \geq \theta \cdot \marginalDistribution\left(\mathcal{A}_{\theta}(\regressionFunction,\noisy[\regressionFunction])\right)>0.
\end{align}
Since $\hat{\classifier}$ is consistent with 
$\Prob_{train}=\Prob_{test}=\Prob_{\left(\marginalDistribution,\noisy[\regressionFunction]\right)}$, we have $\lim_{n \rightarrow \infty}\excessRisk\left(\hat{\classifier}_n;\Prob_{\left(\marginalDistribution,\noisy[\regressionFunction]\right)}\right) = 0$. Hence, from eq. (\ref{ineqApplyingInconsistencyLemma}) it follows that $\limsup_{n \rightarrow \infty}\excessRisk\left(\hat{\classifier}_n;\Prob_{\left(\marginalDistribution,{\regressionFunction}\right)}\right) \geq \theta \cdot \marginalDistribution\left(\mathcal{A}_{\theta}(\regressionFunction,\noisy[\regressionFunction])\right)>0$. 
That is, $\hat{\classifier}$ is inconsistent when trained with train distribution  $\Prob_{\left(\marginalDistribution,\noisy[\regressionFunction]\right)}$ and tested on distribution  $\Prob_{\left(\marginalDistribution,{\regressionFunction}\right)}$.
\end{proof}

It remains to prove the lemma used in the proof above. 
\begin{lemma}\label{inconsistencyLemma} 
Let $\marginalDistribution$ be a Borel probability measure on $\X$. 
Given $\regressionFunction:\X \rightarrow [0,1]$ and $\theta>0$ consider 
the set $\mathcal{A}_{\theta}(\regressionFunction,\noisy[\regressionFunction]) \subseteq \X$ as defined in eq. (\ref{thetaDisagreement}).
Then given any classifier $\classifier: \X \rightarrow \{0,1\}$ we have $
\excessRisk\left({\classifier};\Prob_{\left(\marginalDistribution,\regressionFunction\right)}\right)+\excessRisk\left({\classifier};\Prob_{\left(\marginalDistribution,\noisy[\regressionFunction]\right)}\right) \geq \theta \cdot \marginalDistribution\left(\mathcal{A}_{\theta}(\regressionFunction,\noisy[\regressionFunction])\right)$.
\end{lemma}
\begin{proof} Recall that, for any regression function $\tilde{\regressionFunction}:\X \rightarrow [0,1]$ the excess risk can be written as: $\excessRisk\left(\classifier;\Prob_{\left(\marginalDistribution,\tilde{\regressionFunction}\right)}\right) = $
{\scriptsize
\begin{align}\label{excessRiskEqualTo}
\int \left| \tilde{\regressionFunction}(x)-\frac{1}{2}\right| \cdot \one \left\lbrace  \left(\tilde{\regressionFunction}(x)-\frac{1}{2}\right)\left(\classifier(x)-\frac{1}{2}\right)<0\right\rbrace d\marginalDistribution(x).
\end{align}}
Now if $x \in \mathcal{A}_{\theta}(\regressionFunction,\noisy[\regressionFunction])$ then $\left(\regressionFunction(x)-\frac{1}{2}\right)\left(\noisy[\regressionFunction](x)-\frac{1}{2}\right)<0$ so for both possible values $\classifier(x) \in \{0,1\}$ we have
\begin{align*}
\one \left\lbrace  \left({\regressionFunction}(x)-\frac{1}{2}\right)\left(\classifier(x)-\frac{1}{2}\right)<0\right\rbrace +  \one \left\lbrace  \left(\noisy[\regressionFunction](x)-\frac{1}{2}\right)\left(\classifier(x)-\frac{1}{2}\right)<0\right\rbrace = 1.
\end{align*}
Moreover, if $x \in \mathcal{A}_{\theta}(\regressionFunction,\noisy[\regressionFunction])$ then $\min \left\lbrace \left|{\regressionFunction}(x)-\frac{1}{2}\right|,\left|\noisy[\regressionFunction](x)-\frac{1}{2}\right| \right\rbrace\geq \theta$ and so 
\begin{align}
\left|{\regressionFunction}(x)-\frac{1}{2}\right|\cdot \one \left\lbrace  \left({\regressionFunction}(x)-\frac{1}{2}\right)\left(\classifier(x)-\frac{1}{2}\right)<0\right\rbrace + \left|\noisy[\regressionFunction](x)-\frac{1}{2}\right|\cdot  \one \left\lbrace  \left(\noisy[\regressionFunction](x)-\frac{1}{2}\right)\left(\classifier(x)-\frac{1}{2}\right)<0\right\rbrace \geq \theta.
\end{align}
Integrating with respect to $\marginalDistribution$ and applying (\ref{excessRiskEqualTo}) to both $\Prob_{\left(\marginalDistribution,\regressionFunction\right)}$ and $\Prob_{\left(\marginalDistribution,\noisy[\regressionFunction]\right)}$ gives the conclusion of the lemma.
\end{proof}

\newpage
\section{Proof of Lemma \ref{elementaryRatioLemmaForCorruptedRegressionFunctionEstimation}}\label{proofOfelementaryRatioLemmaForCorruptedRegressionFunctionEstimationAppendix}

\begin{proof}
 Given $\hat{a},a \in [-1,1]$, $b,\hat{b} >0$ with $|\hat{b}-b| \leq b/2$, and $a/b \in [0,1]$,
\begin{align}\label{elementaryRatioIneqs}
\left|\frac{\hat{a}}{\hat{b}}-\frac{a}{b}\right|= \frac{1}{\hat{b}}\cdot \left| (\hat{a}-a)+\frac{a}{b} \cdot (b-\hat{b})\right| \leq \frac{2}{b}\left( |\hat{a}-a|+ \left| \frac{a}{b}\right| \cdot |\hat{b}-b| \right) \leq  \frac{4}{b}\cdot \max\left\lbrace |\hat{a}-a|, |\hat{b}-b| \right\rbrace,
\end{align}
where we have used the fact that $\hat{b} \geq b/2$.
By the definition of $\est[\regressionFunction](x)$ together with eq. (\ref{noisyRegInTermsOfTrueRegEq}) we have
\begin{align*}
\est[\regressionFunction](x):= \frac{\noisy[{\est[\regressionFunction]}](x)-\estProbFlipZeroToOne}{1-\estProbFlipZeroToOne-\estProbFlipOneToZero}\hspace{1cm}\text{and}\hspace{1cm}
\regressionFunction(x) = \frac{\noisy[\regressionFunction](x)-\probFlipZeroToOne}{1-\probFlipZeroToOne-\probFlipOneToZero}.
\end{align*}
Now take $\hat{a} = \estNoisyRegressionFunction(x)-\estProbFlipZeroToOne$, $a =\noisyRegressionFunction(x)-\probFlipZeroToOne$, $\hat{b} = 1-\estProbFlipZeroToOne-\estProbFlipOneToZero$ and $b = 1-\probFlipZeroToOne-\probFlipOneToZero$. Given the assumptions that  $\probFlipZeroToOne+\probFlipOneToZero<1$, so $b>0$ and $\max\left\lbrace\left| \estProbFlipZeroToOne-\probFlipZeroToOne \right|, \left| \estProbFlipOneToZero-\probFlipOneToZero \right| \right\rbrace  \leq \left(1-\probFlipZeroToOne-\probFlipOneToZero\right)/4$ this implies
\begin{align*}
|\hat{b}-b| = 2\cdot \max\left\lbrace\left| \estProbFlipZeroToOne-\probFlipZeroToOne \right|, \left| \estProbFlipOneToZero-\probFlipOneToZero \right| \right\rbrace \leq \frac{1}{2}\cdot \left( 1-\probFlipZeroToOne-\probFlipOneToZero\right) = \frac{b}{2},
\end{align*}
which also implies $\hat{b}\geq b/2 >0$. Hence, by (\ref{elementaryRatioIneqs}) we deduce
\begin{align*}
\left| \est[\regressionFunction](x)-\regressionFunction(x)\right| \leq \frac{4}{1-\probFlipZeroToOne-\probFlipOneToZero} \cdot \max\left\lbrace \left| \left( \estNoisyRegressionFunction(x)-\estProbFlipZeroToOne\right)- \left(\noisyRegressionFunction(x)-\probFlipZeroToOne\right) \right| ,\left| \left(1-\estProbFlipZeroToOne-\estProbFlipOneToZero \right)- \left(1-\probFlipZeroToOne-\probFlipOneToZero\right) \right| \right\rbrace\\
\leq \frac{8}{1-\probFlipZeroToOne-\probFlipOneToZero} \cdot \max\left\lbrace \left|  \estNoisyRegressionFunction(x)-\noisyRegressionFunction(x) \right| ,\left| \estProbFlipZeroToOne- \probFlipZeroToOne \right|,\left| \estProbFlipOneToZero-\probFlipOneToZero \right| \right\rbrace.
\end{align*}
This completes the proof of Lemma \ref{elementaryRatioLemmaForCorruptedRegressionFunctionEstimation}.
\end{proof}

\end{document}